\documentclass[runningheads]{llncs}
\usepackage[T1]{fontenc}
\usepackage{graphicx}
\usepackage{subcaption}
\usepackage{amsmath,amsfonts,amssymb}

\usepackage[ruled,vlined,linesnumbered]{algorithm2e}
\begin{document}
\title{MUC-G4: Minimal Unsat Core-Guided Incremental Verification for Deep Neural Network Compression}
\titlerunning{Incremental Verification for DNN Compression}
% If the paper title is too long for the running head, you can set
% an abbreviated paper title here
% \author{Anonymous Authors}
% \institute{}
\author{Jingyang Li \and Guoqiang Li}
\authorrunning{Jingyang, Guoqiang, et al.}
\titlerunning{}
%
% First names are abbreviated in the running head.
% If there are more than two authors, 'et al.' is used.
%
\institute{Shanghai Jiao Tong University, Shanghai 200240, China\\
\email{\{lijjjjj, li.g\}@sjtu.edu.cn}}

\maketitle              % typeset the header of the contribution
\begin{abstract}
The rapid development of deep learning has led to challenges in deploying neural networks on edge devices, mainly due to their high memory and runtime complexity. Network compression techniques, such as quantization and pruning, aim to reduce this complexity while maintaining accuracy. However, existing incremental verification methods often focus only on quantization and struggle with structural changes. This paper presents MUC-G4 (Minimal Unsat Core-Guided Incremental Verification), a novel framework for incremental verification of compressed deep neural networks. It encodes both the original and compressed networks into SMT formulas, classifies changes, and use \emph{Minimal Unsat Cores (MUCs)} from the original network to guide efficient verification for the compressed network. Experimental results show its effectiveness in handling quantization and pruning, with high proof reuse rates and significant speedup in verification time compared to traditional methods. MUC-G4 hence offers a promising solution for ensuring the safety and reliability of compressed neural networks in practical applications.

\end{abstract}

\section{Introduction}\label{intro}

%网络压缩有可能引入对抗样本
In recent years, the field of deep learning has witnessed an exponential growth in the complexity and size of models, presenting significant challenges when it comes to their deployment on small devices. Quantization~\cite{DBLP:journals/corr/HanMD15} and network pruning~\cite{DBLP:conf/iclr/LiuSZHD19} have emerged as essential network compression techniques to address these challenges by reducing the storage and computational requirements of deep learning models, without causing significant performance loss. Unfortunately, security vulnerabilities or issues can be introduced during network compression. For instance, models may be vulnerable to adversarial examples~\cite{2013Intriguing}. The commonly used empirical estimation and adversarial example-finding techniques~\cite{2015Explaining,2013Intriguing,2017Towards,2017Towards1} can only provide insights into the average performance of the models but cannot guarantee consistent behavior across all inputs. The issue of (non-)robustness in neural networks~\cite{DBLP:journals/corr/HuangKWW16}, which has received increasing attention in the research community, emphasizes this concern.

%神经网络验证可以为网络提供形式化保证，但是在网络压缩中只能一个个验证
In order to improve the trustworthiness of neural networks, existing studies in neural network verification~\cite{Katz2017ReluplexAE,2017Verifying,2017An} have incorporated formal verification techniques into neural network evaluation and have made significant efforts to provide formal guarantees. Neural network verification typically focuses on verifying specifications for fixed network scenarios instead of incrementally verifying properties as the network changes. These methods verify the two networks in network compression separately from scratch, which consumes a large amount of computational resources and time when the preservation of the specification during network compression is concerned. 

%目前的增量验证基本只关注参数改变场景
Incremental verification is a formal technique that aims to verify only the modified portions of a system under changes, rather than rechecking the entire system from scratch. In the context of deep learning, this approach has been successfully employed in conjunction with backend methods like Reluplex or BaB. However, existing research mainly focuses on network parameter changes and overlooks the potential structural changes that frequently occur during network compression. Consequently, efficiently and accurately reusing proof for a structurally altered network in network compression poses a significant challenge. 

%这篇工作中我们提出了一个基于CDCL思想的新框架***，一句话简要介绍CDCL，因此，我们借用这个思想，并尝试将压缩前的网络验证过程中的信息利用起来对压缩后的网络验证过程进行剪枝。
In this work, we propose a novel SMT solving based incremental verification framework designed for neural network compression problem that incorporates the idea of \emph{conflict driven clause learning (CDCL)}. The CDCL framework~\cite{DBLP:series/faia/0001LM21}, in the context of SAT problem solving, is a powerful approach to drive the search for a satisfying assignment by iteratively analyzing the conflict between the current assignment and the constraints. Thus, we leverage the conflict analyzed during the verification of the original network to prune the search space in the verification of the compressed network, reducing the computational complexity and improving the efficiency of the verification process.

To understand the CDCL framework's advantage, consider a neural network $f$ that is first verified and then compressed to $f'$. Existing incremental verification methods typically reuse the initial subproblem generation process and sequence, assuming that the initial heuristic information can be directly applied to $f'$ without any adaptation. However, during network pruning, when neurons or connections are removed, the relationships among network components change. As a result, the original subproblems may not be applicable to $f'$. The computational complexity and importance of subproblems vary, and some original subproblems may even be missing in $f'$. These drawbacks call for a better approach that can dynamically select and prioritize subproblems during the verification of $f'$ by analyzing the conflicts identified during the initial verification of $f$, which we found similar to the idea hidden in the CDCL framework.

In our work, we encode the original network $f$ and the compressed network $f'$ as SMT formulas. Since $f'$ is derived from $f$ through quantization or pruning, the neuron count in $f'$ is either equivalent to (in quantization) or less than (in pruning) that in $f$. Each neuron in $f'$ has a natural correspondence with a neuron in $f$, which enables us to encode the SMT formulas for $f'$ based on those of $f$.
For both quantization and network pruning, we explicitly express $f'$ using SMT formulas. In line with the CDCL concept of analyzing conflicts, during the verification of $f$, if a particular combination of neuron activations leads to a contradiction with the network's overall behavior (such as violating the expected output range), we mark this combination as part of the conflict. By further refining these conflicts, we isolate the \emph{Minimal Unsat Cores (MUCs)}, which is the smallest non-redundant conflict units in infeasible paths. These cores are then checked and used to eliminate similar infeasible paths during the verification of $f'$, reducing the computational complexity.

To validate our MUC-guided incremental verification framework based on CDCL concept, we build a custom SMT solver as the backend neural network verifier. This self-implemented SMT solver uses an ordinary LP solver for theory solving. This way, we avoid potential hybrid effects from other neural network verifiers' search strategies. Our framework is adaptable and can integrate with any other neural network verifiers as its theory solver, showing its versatility.
To improve the verification process's scalability, we tackle the exponential explosion in SMT-based neural network verification with linear relaxation. We also use a heuristic method to choose the best neurons for ``unrelaxing'' to maintain completeness.
Finally, we conduct a performance evaluation using the ACAS Xu benchmark, MNIST-trained neural networks, and their compressed versions. The results show our approach is effective. It provides high quality proofs and speeds up verification time, validating the practicality and efficiency of our framework.

\smallskip\noindent \textbf{Contributions:} To summarize, our contributions are:    
\begin{enumerate}
\item We introduce MUC-G4, the first MUC-guided framework that unifies incremental verification for both quantization and pruning in neural network compression, enabling systematic handling of compression-related changes.

\item We propose a heuristic strategy that leverages MUCs from the original network's SMT solving process to optimize search paths, significantly reducing verification effort.

\item We implement MUC-G4 with Z3 SAT solver and an LP theory solver, demonstrating its effectiveness through experiments on ACAS-Xu and MNIST by achieving high proof validity and accelerated verification times.
\end{enumerate}

\section{Preliminaries}\label{pre}
\subsection{Neural Network and Notations}
A deep neural network (DNN) represents a function $f: I\rightarrow O$, where $I\subseteq \mathbb{R}^n$ denotes an input domain with $n$ features, and $O\subseteq \mathbb{R}^m$ corresponds to an output domain with $m$ dimensions. It consists of various layers, including an input layer, multiple hidden layers, and an output layer. An example of a hidden layer is the fully connected layer. 
\begin{figure}[htbp]
    \centering
    \includegraphics[width=0.4\linewidth]{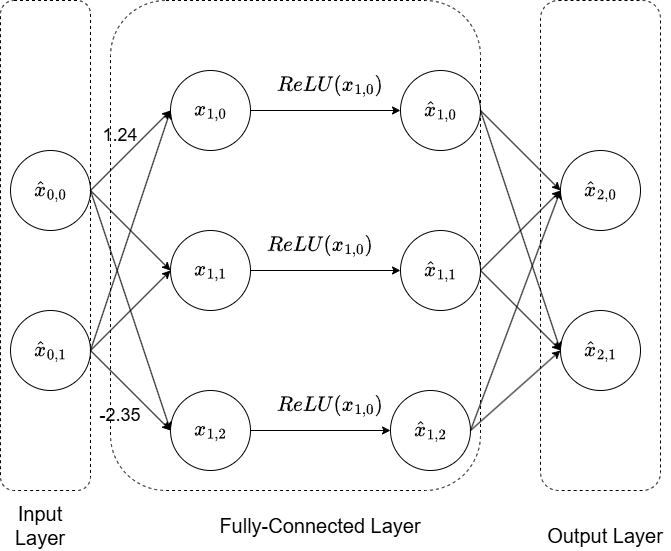}
    \caption{An example of a DNN $f$ containing one input layer, one fully-connected layer and one output layer.}
    \label{fig:DNN}
\end{figure}
In a fully connected layer, as depicted as in Fig. \ref{fig:DNN}, each neuron's calculation is based on the weighted sum of the neurons' outputs from the previous layer followed by an activation function. This means that the input from each neuron in the previous layer is multiplied by a weight (as indicated by the values like 1.24 and -2.35 in the figure), and then these weighted inputs are summed up and added with a bias (we assume to be 0 and thus ignore it in the figure). An activation function, such as the ReLU function shown in the figure, is then applied to this sum to introduce non-linearity into the model. The output is then passed on to the next layer. Note that the weights and the bias are all trainable parameters. Specifically, we focus on the ReLU (Rectified Linear Unit) function, expressed as $\mathrm{ReLU}(x) = \max(0, x)$, in our paper.

In this paper, we represent the set of positive integers up to $L$ as $[L]$. The weight matrix of the $k$-th layer is denoted as $\mathbf{W}_k$. The input of the $j$-th node in the $k$-th layer is referred to as $x_{k,j}$ while the output of the $j$-th node in the $k$-th layer is referred to as $\hat{x}_{k,j}$. The weight of the edge from $\hat{x}_{k-1,i}$ to $x_{k,j}$ is represented as $W_k[i,j]$. We use bold symbols like $\mathbf{x}_k$ and $\mathbf{W}_k$ for vectors and matrices, and regular symbols like $x_{k-1,i}$ and $W_k[i,j]$ for scalars. 

We define an ReLU neural network $f$ with an input layer, $L-1$ hidden layers and an output layer. The $i$-th hidden layer of the network contains $n_i$ neurons. The network $f$ is defined by its weight matrices $\mathbf{W}_k$ and bias vectors $\mathbf{b}_k$, where $k$ ranges from $1$ to $L$. By convention, the input layer is referred to as layer $0$. The neural network calculates the output $\mathbf{x}_L$  for an input $\hat{\mathbf{x}}_0$ by recursively applying affine transformation: $\mathbf{x}_k = \mathbf{W}_k\hat{\mathbf{x}}_{k-1}+\mathbf{b}_k$, and ReLU function: $\hat{\mathbf{x}}_k=\max(\mathbf{0},\mathbf{x}_k)$.

Due to the nature of the ReLU function, given an input $\mathbf{x}$, any intermediate neurons can be \textbf{activated} or \textbf{deactivated}. Specifically, if the input to a neuron is positive, the neuron is activated, meaning its output is typically equal to the input. Conversely, if the input is negative, the neuron is deactivated, and its output becomes zero. This piecewise linear behavior introduces a non-linearity in the network, allowing it to model complex functions. Furthermore, neurons that remain consistently active or inactive across a range of inputs can be considered \textbf{stable}, while neurons performs otherwise are denoted as \textbf{unstable}. 

\subsection{Neural Network Compression}
\begin{figure}[htbp]
    \centering
    \begin{minipage}[b]{0.45\linewidth} 
        \centering
        \includegraphics[width=\linewidth]{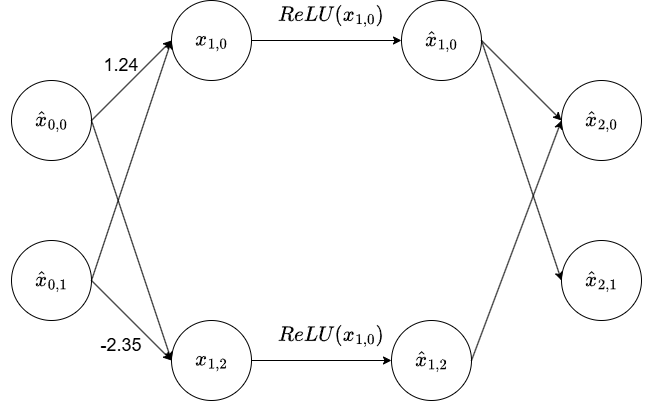}
        \caption{A pruned DNN $f_p$}
        \label{fig:pruned DNN}
    \end{minipage}
    \hfill
    \begin{minipage}[b]{0.45\linewidth}
        \centering
        \includegraphics[width=\linewidth]{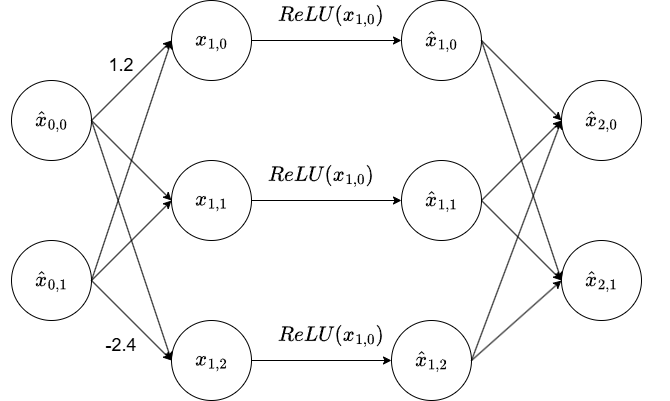}
        \caption{A quantized DNN $f_q$}
        \label{fig:quantized DNN}
    \end{minipage}
    
\end{figure}
DNNs are often constrained by their size and computational complexity, posing challenges for their deployment on edge devices. Consequently, various network compression methods have been introduced. One prominent approach is model quantization~\cite{DBLP:journals/corr/HanMD15,DBLP:conf/iccv/GongLJLHLYY19}. Model quantization involves compressing the parameters of a neural network, which are initially represented as floating-point values, into fixed-point (integer) representation. 
%The fixed-point data is subsequently dequantized back into floating-point data to obtain the final result. 
Model quantization can accelerate model inference and reduce model memory usage. It exploits the network's tolerance to noise, introducing a certain amount of noise (quantization error) that does not significantly impact accuracy if the number of quantization bits is appropriately chosen. 
An example is depicted in Fig. \ref{fig:quantized DNN}, where the weight from neuron $\hat{x}_{0,0}$ to neuron $x_{1,0}$ is quantized from $1.24$ to $1.2$ and the weight from neuron $\hat{x}_{0,1}$ to neuron $x_{1,2}$ is quantized from $-2.35$ to $-2.4$. 

Another significant class of compression techniques for neural networks is network pruning~\cite{DBLP:conf/cvpr/MolchanovMTFK19,DBLP:conf/nips/ChoAN23,DBLP:conf/iclr/LiuSZHD19}. Neural network pruning aims to eliminate redundant parts of a network that consume substantial resources while maintaining good performance. Despite the large learning capacity of large neural networks, not all components are actually useful after the training process. Neural network pruning leverages this fact to remove these unused parts without compromising the network's performance. An example is depicted in Fig. \ref{fig:pruned DNN}, where the neuron $x_{1,1}$ is pruned and the edge from $\hat{x}_{1,2}$ to $x_{2,2}$ is also pruned.

% \begin{figure}[htbp]
%     \centering
%     \includegraphics[width=0.5\linewidth]{Figure/prunedDNN.png}
%     \caption{A hidden layer of a pruned DNN $f_p$}
%     \label{fig:pruned DNN}
% \end{figure}

% Formally, given a neural network $f$ and its pruned version $f'$, network pruning is equivalent to the optimization problem:
% \begin{align*}
%     \max_{f'}\quad &Acc(f',D) \\
%      \text{s.t. } \quad &Budget(f\leftarrow f')\leq C
% \end{align*}
% , where $Acc(f',D)$ represents the accuracy of pruned network $f'$ on dataset $D$ and $Budget(f-f')$ represents the budget of the pruning process. 

% It should be noted that the majority of current quantization techniques do not modify the network architectures, thereby enabling the formal verification of accuracy loss in quantization through existing differential verification tools. Conversely, with existing network pruning techniques, empirical evaluation of accuracy loss predominates, lacking formal guarantees. Unfortunately, this scenario remains outside the purview of existing differential verification tools as they have yet to extend themselves to accommodate such cases.
\subsection{Neural Network Verification}
Neural network verification employs formal verification techniques to mathematically prove the properties of neural networks. In this approach, a deep neural network is considered a mapping function from input to output, and the property is formalized as a logical formulation involving the inputs and outputs. Formally, neural network verification is defined as follows:
\begin{definition}[Neural Network Verification]\label{nnv}
Let $f: I\rightarrow O$ be a neural network and consider two predicates $P$ and $Q$. The verification problem $(f,P,Q)$, pertaining to verifying $f$ with respect to $P$ and $Q$, is to determine whether
\begin{equation*}
    \forall \mathbf{x}\in I, P(\mathbf{x})\Rightarrow Q(f(\mathbf{x})).
\end{equation*}
holds. Alternatively, the goal is to demonstrate that:
\begin{equation*}
    P(\mathbf{x})\wedge \neg Q(f(\mathbf{x}))
\end{equation*}
is unsatisfiable (UNSAT) or not. In instance where it is satisfiable (SAT), a counterexample $\mathbf{x}_{adv}$ can be found, such that $P(\mathbf{x}_{adv})$ holds true and $Q(f(\mathbf{x}_{adv}))$ does not.
\end{definition}
Note that, in general scenarios, the input property 
$P$ is often interpreted as a region of interest within the input space, while $Q$ typically specifies desired behaviors or constraints on the output. For example, $Q$ might impose a condition such as a specific dimension of the output always being the largest. Neural network verification ensures the absence of defects or adherence to specific properties.  Various strategies can be employed within the existing modeling framework to verify this property description.

\subsection{Satisfiability Modulo Theories Solving}

Satisfiability Modulo Theories (SMT) solving involves determining the satisfiability of logical formulas that incorporate constraints from various mathematical theories, such as arithmetic, bit-vectors, arrays, and data structures~\cite{DBLP:series/faia/BarrettSST09}. Diverging from the scope of satisfiability (SAT) solving, which addresses formulas expressed solely through Boolean variables, SMT solving broadens the ambit to encompass richer theories, rendering it apt for modeling and discerning intricate systems. The efficacy of SMT solving hinges upon the amalgamation of methodologies drawn from mathematical logic, computer science, and formal methods. Diverse algorithmic paradigms have been devised to expedite the resolution of SMT instances. These methodologies typically entail a fusion of SAT solving techniques with theory-specific decision procedures. Notably, CDCL (Conflict-Driven Clause Learning) is a key advancement in SAT solvers~\cite{DBLP:series/faia/0001LM21}, enhancing their efficiency and capability to tackle complex instances. In a CDCL solver, when a conflict arises (i.e., the current assignment of variables leads to a contradiction), the solver analyzes the conflict to learn a new clause that prevents the same assignment from recurring. This process is crucial for pruning the search space, allowing the solver to explore more promising areas of the solution space.

The utility of SMT solving transcends numerous domains. Confronted with the rapid development of deep learning, SMT solving has found application in the realm of neural network verification, such as checking the correctness of neural networks with respect to specified safety properties or invariants. Despite the temporal requirements associated with SMT-based neural network verification methods, they provide formal guarantees in a sound and complete manner. In this paper, we adopt linear arithmetic as our backend theory.

Concerning the verification procedure, conventionally, the negation of the desired property is encoded into the SMT formulas. Consequently, an UNSAT result is yielded if the problem does not admit any solution, thereby ensuring the property $Q$ of $f$ under any input that satisfies $P$. Conversely, a SAT outcome is obtained upon discovering a solution for the provided formulas. Furthermore, an UNK outcome is returned if no solution is found within a specified time frame while the problem is not entirely checked.

\section{Problem Definition and Overview}\label{overview}
\subsection{Problem Definition}
The incremental verification problem in DNN compression is defined as follows:
\begin{definition}[Incremental Verification for Compressed Neural Network]
    Let $f: I\rightarrow O$ be a neural network, and two predicates $P$ and $Q$. The verification problem $(f,P,Q)$ is first proven with its proof denoted as $p(f,P,Q)$. Consider the compressed neural network $f'$ of $f$. The incremental verification problem for compressed neural network $f'$ is defined as the verification problem $(f',P,Q)$ using the proof $p(f,P,Q)$.
\end{definition}
\subsection{Overview}
The MUC-G4 framework proposed in this paper aims to address the incremental verification problem during the compression of deep neural networks. 
The workflow of our framework is depicted in Fig. \ref{fig:framework}. 
\begin{figure}[t]
    \centering
    \includegraphics[width=0.7\linewidth]{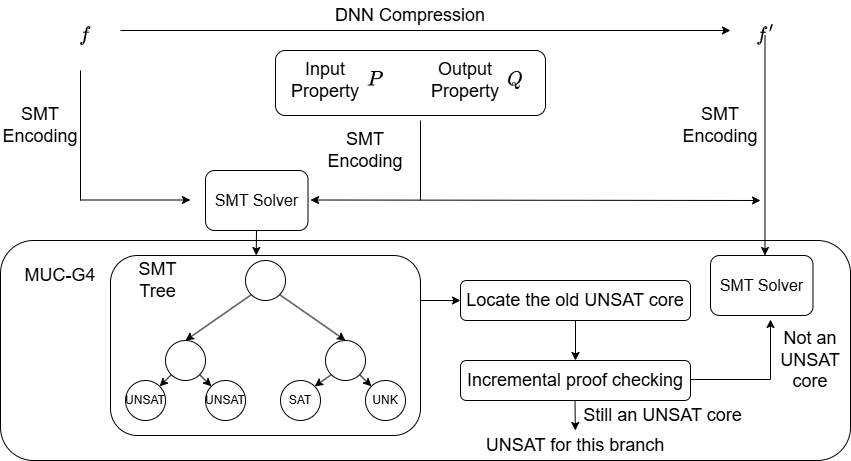}
    \caption{The framework of incremental verification in DNN compression using MUC-G4.}
    \label{fig:framework}
\end{figure}
Given a verification problem $(f,P,Q)$, we initially encode it into SMT formulas. The SMT solver then processes these formulas, generating an SMT search tree. Taking the neural network $f$ in Fig. \ref{fig:DNN} as an example, Fig. \ref{fig:SMT Tree} shows the corresponding SMT search tree. 
Starting from the root node, the two edges of each internal node are respectively labeled with two different activation state assertions of a certain neuron $x_{k,i}$ in $f$:  $x_{i,j}\geq0\wedge \hat{x}_{i,j}=x_{i,j}$ and $x_{i,j}<0\wedge \hat{x}_{i,j}=0$. These two assertions represent the activated and deactivated states of the neuron, respectively. Each path from the root node to a leaf node in the search tree corresponds to a combination of neuron activation states. 
%The SMT formulas formed by these paths limit the behavior of uncertain ReLU neurons on that path, determining whether they are activated or deactivated. 
The leaf nodes are marked with the verification results obtained by solving with the SMT solver, including SAT, UNSAT, or UNK. 
\begin{figure}
    \centering
    \includegraphics[width=0.9\linewidth]{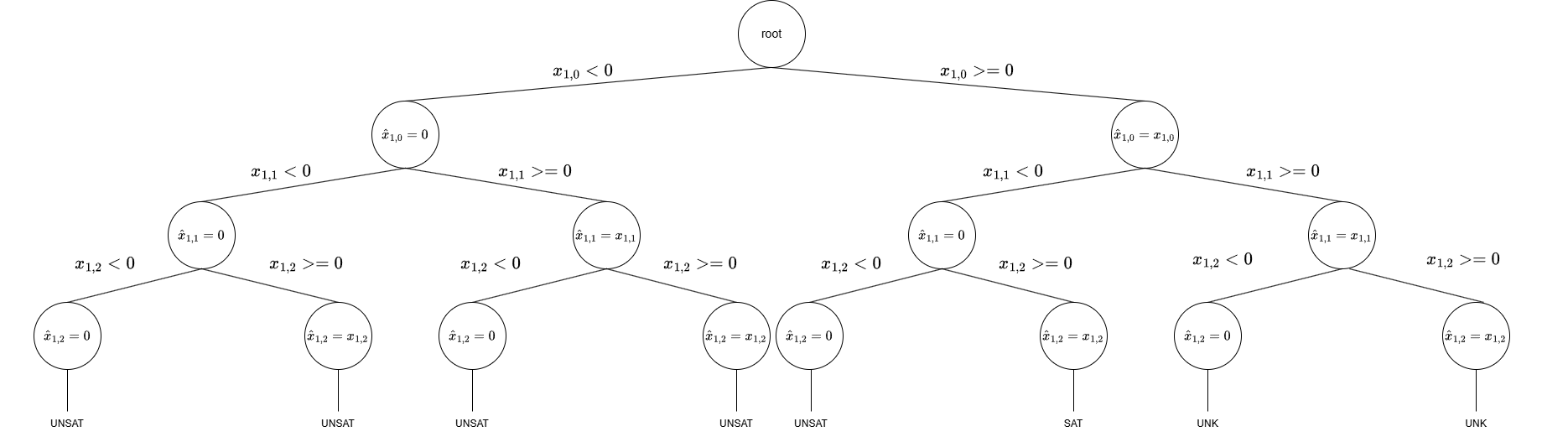}
    \caption{The SMT tree conducted using $f$ in Fig. \ref{fig:DNN}.}
    \label{fig:SMT Tree}
\end{figure}

We now explain how to transform the SMT tree into the proof in Fig. \ref{fig:SMT Proof}. 
\begin{figure}
    \centering
    \includegraphics[width=0.7\linewidth]{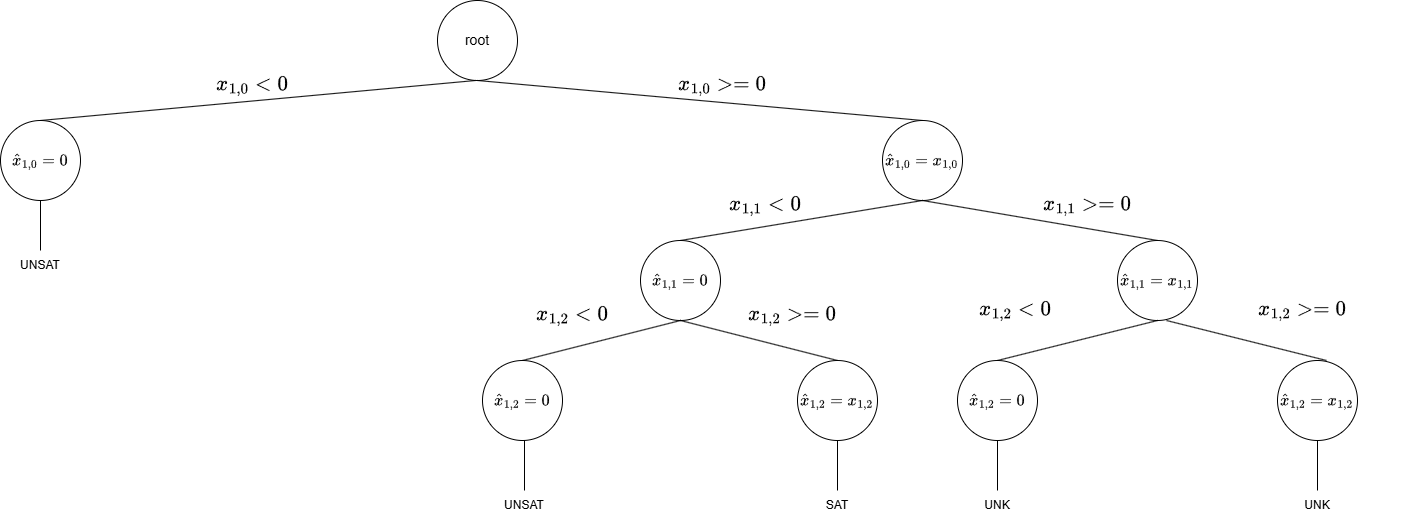}
    \caption{The SMT proof conducted with Fig. \ref{fig:SMT Tree}.}
    \label{fig:SMT Proof}
\end{figure}
We start by traversing the SMT tree in Fig. \ref{fig:SMT Tree} from the root to the leaf nodes. For each path that leads to a leaf node marked as UNSAT, we record the sequence of neuron activation state assertions along that path. These assertions are the ones labeling the edges of the tree. Next, we simplify the recorded assertions for each UNSAT path. We look for redundant assertions within the set of assertions that lead to the UNSAT result. For example in Fig. \ref{fig:SMT Tree}, we initially discover that $x_{1,0}<0$ and $x_{1,1}<0$ and the root are sufficient to obtain an UNSAT result, meaning the status of $x_{1,2}$ is not necessary to cause the overall contradiction in the path. Thus we remove it from this path and we continue to conduct the removal. Eventually we obtain the path containing only $x_{1,0}<0$ as in Fig. \ref{fig:SMT Proof}. After constructing the proof in Fig. \ref{fig:SMT Proof}, we check the Minimal Unsat Cores for the compressed network $f'$ to see if the cores remain UNSAT after the compression. By doing so, we prune the search space and guide the search for the verification of $f'$.

\section{SMT Formula Solving-Based Incremental Verification}\label{approach}
% In this section, we describe our approach for incrementally verifying compressed network efficiently and accurately. We first present the construction of the useful proof within the verification of the original network. Then we introduce the strategies on reusing the proof that can reduce the SMT assertions needed to check and increase scalability without harming soundness and completeness.

\subsection{Basic Algorithm}
Given a neural network verification problem $(f,P,Q)$, in order to solve it with SMT solvers, it is needed to encode the problem as SMT formulas. The encoding involves determining the variables and the constraints in the formulas. 
\subsubsection{\textbf{Variables}}
Using the notations in previous section, the variables in the equivalent SMT formulas of $f$ with an input layer, $L-1$ hidden layers and an output layer are constructed using the input and the output of the neurons in $f$. The input of $f$ can be naturally encoded into $n_0$ input variables $\hat{x}_{0,1},...,\hat{x}_{0,n_0}$. Similarly, the output of $f$ can be encoded into $n_L$ output variables $x_{L,1},...,x_{L,n_L}$. Besides the input and output variables of $f$, its intermediate neurons should also be encoded as variables. Given hidden layer $i$ with $0<i<L$, the associated variables of its input are denoted as $x_{i,1},...x_{i,n_i}$, and the associated variables of its output are denoted as $\hat{x}_{i,1},...\hat{x}_{i,n_i}$.
\subsubsection{\textbf{Constraints}}
The constraints in the formula can also be categorized into input constraints (referred to as $P$ in Definition \ref{nnv}), output constraints (referred to as $Q$ in Definition \ref{nnv}) and intermediate constraints. 

The inputs are often constrained to lie between given lower and upper bounds. The input constraints are as follows:
\begin{equation}
    \bigwedge^{n_0}_{j=1} l_j\leq \hat{x}_{0,j}\leq u_j
\end{equation}
, where $l_j,u_j\in\mathbb{R}$ denote the lower and upper bounds for the input feature $\hat{x}_{n_0,j}$. As for the intermediate neurons, the constraints are two folds. On the one hand, for any $i\in[1,L]$ there exist constraints encoding the affine transformation:
\begin{equation}
    \bigwedge^{n_i}_{j=1}(x_{i,j}=\Sigma^{n_{i-1}}_{k=1}\hat{x}_{i-1,k}W_i[k,j]+b_{i,j})
\end{equation}
. On the other hand, for any $i\in[1,L-1]$, the ReLU activation function can be encoded as:
\begin{equation}
    \bigwedge^{n_i}_{j=1}(x_{i,j}\geq0\wedge \hat{x}_{i,j}=x_{i,j})\vee (x_{i,j}<0\wedge \hat{x}_{i,j}=0)
\end{equation}

When it comes to the constraints on the output variables, the form is usually dependent on the exact output property. The output constraints are often constructed according to the negation of $Q$ in the verification problem $(f,P,Q)$, which we refer to as $\neg Q$. Hence, solving the conjunction of these SMT formulas and $\neg Q$ is equivalent to resolving the verification problem of $f$.
\subsection{Critical Proof Components}
\subsubsection{\textbf{unsat core and Minimal Unsat Core}}
During the verification, when the SMT solver determines that a set of formulas is unsatisfiable, it means that there is no assignment of variables that can satisfy all the constraints simultaneously. This situation often occurs due to conflicts among different parts of the neural network's structure, activation functions, input, and output constraints. To better analyze and manage these unsatisfiable cases, we introduce the concept of the unsat core.

\begin{definition}[Unsat Core]
    Given a set of assertions (constraints) $S$ in an SMT problem, the Unsat Core is a subset $C \subseteq S$ such that the conjunction of all assertions in $C$ is unsatisfiable.
\end{definition}

However, an unsat core may contain redundant information. In particular, given two unsat cores $C_1$ and $C_2$ with $C_1\subseteq C_2$, checking $C_1$ uses less time than $C_2$ since it is shorter. Thus, to obtain a more concise and essential representation of the unsatisfiable part, we further define the Minimal Unsat Core, also denoted as MUC in the following sections. 

\begin{definition}[Minimal Unsat Core]
    Let $C$ be an unsat core. The Minimal Unsat Core $C_{\text{min}}$ is a subset of $C$ such that:
    \begin{enumerate}
        \item $C_{\text{min}}$ is unsatisfiable.
        \item There is no proper subset $C' \subsetneq C_{\text{min}}$ such that $C'$is unsatisfiable.
    \end{enumerate}

\end{definition}

\subsubsection{\textbf{Counterexample and SAT Core}}
In some cases, neural network verification returns a counterexample. A counterexample in neural network verification is a specific input within the defined input domain that violates the desired property being verified. 
When dealing with incremental verification in the context of neural network compression, understanding the behavior of counterexamples becomes crucial. As the network is compressed, the relationships between neurons and the overall behavior of the network change. The counterexamples from the original network may not directly translate to counterexamples in the compressed network. However, they can still provide valuable insights. This is where the concept of the SAT Core comes into play. Expanding the search to the SAT Core is essential because it encompasses both the counterexample and the path containing it that led to the SAT result. SAT core is defined as follow:
\begin{definition}[SAT Core]
    Given an SMT problem, its SAT core is a tuple $( C_{ex}, P_{SAT})$, where  $C_{ex}$ represents its counterexample and $P_{SAT}$ its corresponding path whose verification results is SAT.
\end{definition}

\subsection{Optimization}
% Neural network verification becomes increasingly challenging as the network size grows, particularly due to the presence of a large number of unstable neurons. Traditional methods of handling these neurons, such as directly splitting disjunctions, lead to exponential growth in running time. This is because each disjunctive term represents a different behavior scenario of the neuron, and the number of possible combinations of these scenarios increases exponentially with the number of unstable neurons. To address this computational bottleneck, we implement linear relaxation technique to increase the overall scalability of SMT solving and a heuristic way to decide order of ``unrelaxing'' neurons to restore the completeness violated by linear relaxation. 

\subsubsection{\textbf{Linear Relaxation}}
Notably, splitting the disjunction in (3) for the unstable neurons can lead to an exponential explosion of running time. When splitting the disjunction in (3) for the unstable neurons, we are dealing with multiple possible cases or branches that arise from the logical disjunction. Each disjunctive term represents a different condition or scenario that the neuron's behavior might satisfy. As the number of unstable neurons increases, the number of possible combinations of these different conditions grows exponentially. This rapid growth in the number of combinations to be evaluated leads to an exponential explosion in the running time, making the analysis computationally intractable as the size of the neural network (i.e., the number of unstable neurons) becomes larger.

In order to avoid such an inefficiency, We adopt the linear relaxation techniques and rewrite the formula (3) for unstable neurons as:

\begin{equation}
\bigwedge_{j=1}^{n_i} \hat{x}_{i,j} \geq 0 \wedge \hat{x}_{i,j} \geq x_{i,j} \wedge \hat{x}_{i,j} \leq \frac{u_{i,j}}{u_{i,j} - l_{i,j}} (x_{i,j}-l_{i,j}) 
\end{equation}
Here, $u_{i,j}$ and $l_{i,j}$ are the upper and the lower bounds computed for the $j-th$ neuron in the $i-th$ layer.

Notably, for each unstable neuron, the rewritten formula effectively constrains the output $\hat{x}_{i,j}$ by establishing a set of linear inequalities that capture the relationships between the unstable neuron activations and their corresponding bounds. By establishing these linear inequalities, we avoid having to enumerate and analyze each individual disjunctive case separately. The inequalities work together to define a region in the activation space that over-approximates the original neuron behavior. This is summarized in the following lemma:
\begin{lemma}
    Let $S_{original}$ be the region in the activation space defined by formula (3). Let $S_{relaxed}$ be the region in the activation space defined by the linear inequalities established through formula (4). Then, $S_{original}\subseteq S_{relaxed}$, meaning that the region $S_{relaxed}$ is an over-approximation of the region $S_{original}$.
\end{lemma}
\begin{proof}
    $S_{original}$ is the two boundaries of the triangular region of $S_{relaxed}$.
\end{proof}

Since we adopt linear relaxation in our solving procedure, we summarize in the following lemma about the correctness of the computed unsat cores:
\begin{lemma}
    Given a group of SMT formulas $S_1$ encoded from a neural network verification problem, and $S_2$ being the group of SMT formulas where some of the ReLU constraints in $S_1$ are replaced with their linear relaxation, if $C_1$ is an unsat core of $S_2$, then $C_1$ is an unsat core of $S_1$.
\end{lemma}
\begin{proof}
    This is directly induced by lemma 1.
\end{proof}

This over-approximation allows us to reason about the neuron's behavior in a more holistic and computationally efficient manner, without getting bogged down in the combinatorial explosion that would result from directly splitting the disjunction. In essence, it provides a more efficient way to bound the possible behaviors of the unstable neurons, thus maintaining a balance between comprehensively capturing their behavior and ensuring computational feasibility.

\subsubsection{\textbf{Heuristic Neuron Selection}}
It is worth noting that linear relaxation alone improves the scalability of the neural network verification process at the cost of the completeness. That is, the process might gives false alert in forms of spurious counterexamples because of the over-approximation brought by the linear relaxation of neurons. To avoid the violation of completeness of the overall verification process, it is necessary to refine the whole SMT formulas in such way that the over-approximation become tighter. A straightforward way is to choose a neuron of index $(i,j)$ that is overestimated during the linear relaxation from (3) to (4) and restore it backwards by add the following formulas to (3):
\begin{equation}
(x_{i,j}\geq0\Rightarrow\hat{x}_{i,j}\leq x_{i,j})\wedge(x_{i,j}<0\Rightarrow \hat{x}_{i,j}\leq0)
\end{equation}

Intuitively, the order to select the neurons restore the completeness is important because each refinement operation decreases the overall scalability. Therefore, the aim is to select the neurons that might be ''more over-estimated'' than the other neurons during the linear relaxation. Hence, a possible heuristic optimization is that we score each unstable neuron as follows:
\begin{equation}
    Score(neuron) = \frac{|u+l|}{u+|l|}
\end{equation}
Here, $u$ and $l$ are the upper and the lower bounds of such neuron. According to this scoring order, we assign values to the variables of the SAT solver to determine the state of the neurons and solve the corresponding subproblem. The intuition behind such heuristic variable selection is that split neurons with balanced positive part and negative part can generate a tighter subproblem that over-approximate the original one. Similarly, we guide the conventional backtracking in SMT solving with the same score.
\subsubsection{\textbf{Verification Guidance with On-the-fly MUC Extraction}}
Instead of naively computing MUCs after the completion of the verification, we conduct the MUC computation on-the-fly each time we find an infeasible path. By doing so, we discover that the MUC of a certain path can guide the search of the verification of the remaining paths by ignoring the paths that contain this MUC. 

\subsection{MUC Guided Incremental Verification for Compression}

\subsubsection{\textbf{SMT Formulas in Network Quantization}}
Consider that $f'$ is quantized out of $f$. In this case, two assumptions naturally arise. First is that the network structures are the same, thus the variables in the corresponding SMT formulas for $f'$ is organized the same way in $f$. Second is that $f'$ only slightly differs in the weights and bias in the affine transformation. Hence, given $(f,P,Q)$, the SMT formulas conducted from $(f',P,Q)$ only differ in the affine transformation:
\begin{equation}
    \bigwedge^{n_i}_{j=1}(x_{i,j}=\Sigma^{n_{i-1}}_{k=1}\hat{x}_{i-1,k}W'_i[k,j]+b'_{i,j})
\end{equation}

Notably, the MUCs in the original network $f$ are based on the fundamental behavior of the network, which includes the impact of weights and biases on neuron activations. Although the values have changed slightly in $f'$, the overall patterns of how these parameters interact with neuron activations are likely to be similar. For instance, if in $f$ a certain group of weight values cause a neuron to be activated in a way that led to a conflict in the MUC, in $f'$, the quantized weight may still have a comparable effect on the neuron's activation, and thus, the MUC can provide valuable insights into how this change might affect the overall satisfiability of the network.

\subsubsection{\textbf{SMT Formulas in Network Pruning}}

As for the case where $f'$ is pruned out of $f$, structural changes are made while parameters are assumed to be the same. Note that node pruning is indeed a special case of edge pruning. Hence, in this case, we construct the SMT formulas using an indicator function $\mathbf{1}$:
\begin{equation}
    \bigwedge^{n_i}_{j=1}(x_{i,j}=\Sigma^{n_{i-1}}_{k=1}\hat{x}_{i-1,k}W_i[k,j](1-\mathbf{1})+b_{i,j}
\end{equation}
, where $\mathbf{1}$ is defined as:
\begin{align}
    \mathbf{1}_{i,k,j} = 
\begin{cases} 
1 & \text{if the edge from neuron $i$ in layer $k-1$}\\
&\text{ to neuron $j$ in layer $k$ is pruned} , \\
0 & \text{otherwise }.
\end{cases} 
\end{align}

Unlike heuristic techniques in other incremental verification research, MUC guidance is worth using in the structurally changed compressed network. The reasons are as follows. Even after pruning, a significant portion of the original network structure remains intact. The MUCs from the initial verification were derived from the overall behavior of the network, including the relationships between neurons that might be still present in the pruned network. Besides, although pruning changes the network structure, the MUCs capture the fundamental patterns of unsatisfiability based on neuron activations and constraints. These patterns can still hold true for the pruned network, especially if the pruning is done in a way that does not disrupt the main logical flow of the network.

\subsubsection{\textbf{MUC Guided Incremental Verification}}
We further extend the idea of MUC guided verification to the verification of the compressed network. The algorithm is illustrated in Algorithm \ref{algo:main} in Appendix \ref{algo}. If the old verification result is UNSAT, we simply incrementally check whether the old Minimal Unsat Cores still hold for $f'$, and for an old Minimal Unsat Core that cannot be immediately proved, we start an SMT solving that check the formulas containing this core. If the old verification result is SAT, we first locate ourselves in the SAT path and see whether there exist a counterexample for the compressed network. If not, we go through the unchecked path in the old solving, from the nearest to the farthest path. If still there is no counterexample, we go back to the old UNSAT path and check the old Minimal Unsat Core within.

\section{Implementation and Evaluation} \label{implement}
% We have developed MUC-G4, a framework based on SMT solvers, for verifying neural network compression in an incremental manner. To test its speed improvement, we compared it to the original version. The input for MUC-G4 is a compressed DNN ($f'$) in pytorch format, along with the previous proof of the original DNN ($f$). Our implementation also includes modules that can generate unsat cores or counterexamples for the original network, as explained in Section \ref{approach}. Using these cores and the SMT tree of the original network\ref{approach}, we generate a proof. Additionally, we utilize the z3 solver to incrementally check the cores for the compressed network. The process concludes when the z3 solver either returns a valid counterexample for the compressed network or UNSAT.

\subsection{Benchmarks}

\noindent{\textbf{ACAS-Xu}}
The ACAS Xu benchmarks contain 45 neural networks for testing flight control system safety, which is crucial for aviation safety research and engineering design. Each network has an input layer with 5 inputs, an output layer with 5 outputs, and 6 hidden layers with 300 hidden neurons in total. The five inputs are as follows: distance from the owner to the intruder ($\rho$), heading direction angle of the intruder relative to the owner ($\theta$), speed of the owner ($v_{own}$), and speed of the intruder ($v_{int}$).  The five outputs are: clear of conflict (COC), weak left turn at an angle of $1.5^o/s$ ($weak\ left$), weak right turn at an angle of $1.5^o/s$ ($weak\ right$), strong left turn at an angle of $3.0^o/s$ ($strong\ left$), and strong right turn at an angle of $3.0^o/s$ ($strong\ right$). Each output of the neural network in ACAS Xu corresponds to the score given for a specific action. 

\smallskip\noindent{\textbf{MNIST}}
The MNIST dataset is a collection of labeled images that depict handwritten digits, and it is widely used as a standard benchmark for evaluating the performance of image classification algorithms. Each image in this dataset is composed of 28x28 pixels, with each pixel containing a grayscale value within the range of 0 to 255. When neural networks are trained using this dataset, they are designed to process 784 inputs, with each input representing a pixel's grayscale value falling within the range of 0 to 255. The objective of these neural networks is to produce 10 scores, typically ranging between -10 and 10 in the case of our specific network, corresponding to each of the 10 digits. Subsequently, the digit associated with the highest score is selected as the classification outcome.

\subsection{Experimental Evaluation}
We aim to answer the following research questions:
\begin{enumerate}
\item Is MUC-G4 capable of generating proofs that are ``robust'' in network compression scenarios? Specifically, what is the ratio of the generated unsatisfiable cores that are valid within the context of the compressed network?

\item Is MUC-G4 capable of generating proofs that are critical in network compression scenarios? Specifically, what is the impact on overall scalability?
\end{enumerate}

Consequently, we propose the notion of proof validity ($PV$) to describe how ``robust'' the generated unsat cores are to the compressed network:
\begin{equation}
    Validity(p(f))=\frac{p(f)\cap unsat\_core(f')}{p(f)}
\end{equation}
, where $p(f)$ denotes the proof generated for $f$. It calculates how many unsat cores of $f$ are also unsat cores of $f'$.
Our experiments also compute the proof acceleration ratio to show the speed up in time cost. The input for MUC-G4 is a compressed DNN ($f'$) in pytorch format, along with the previous proof of the original DNN ($f$). Our implementation also includes modules that can generate unsat cores or counterexamples for the original network, as explained in Section \ref{approach}. Besides, the vanilla SMT approach in our paper refers to the SAT+LP approach. The experiments are conducted on a Linux server running Ubuntu 18.04 with an Intel Xeon CPU E5-2678 v3 featuring 12 CPUs.

\subsection{Results}
\subsubsection{\textbf{ACAS-Xu}}
We conducted experiments on ACAS Xu benchmarks to evaluate our proposed MUC-G4 framework on incremental verification task for neural network quantization. Our goal was to verify the variants of the 45 ACAS Xu DNNs used for airborne collision avoidance and evaluate the quality of the proofs generated by our approach. In order to take all 45 networks into comparison, we select the property $\phi_1$ and $phi_2$ in the benchmark. 

\begin{figure}[!t]
    \centering
    \begin{subfigure}[b]{0.45\textwidth}
        \includegraphics[width=\textwidth]{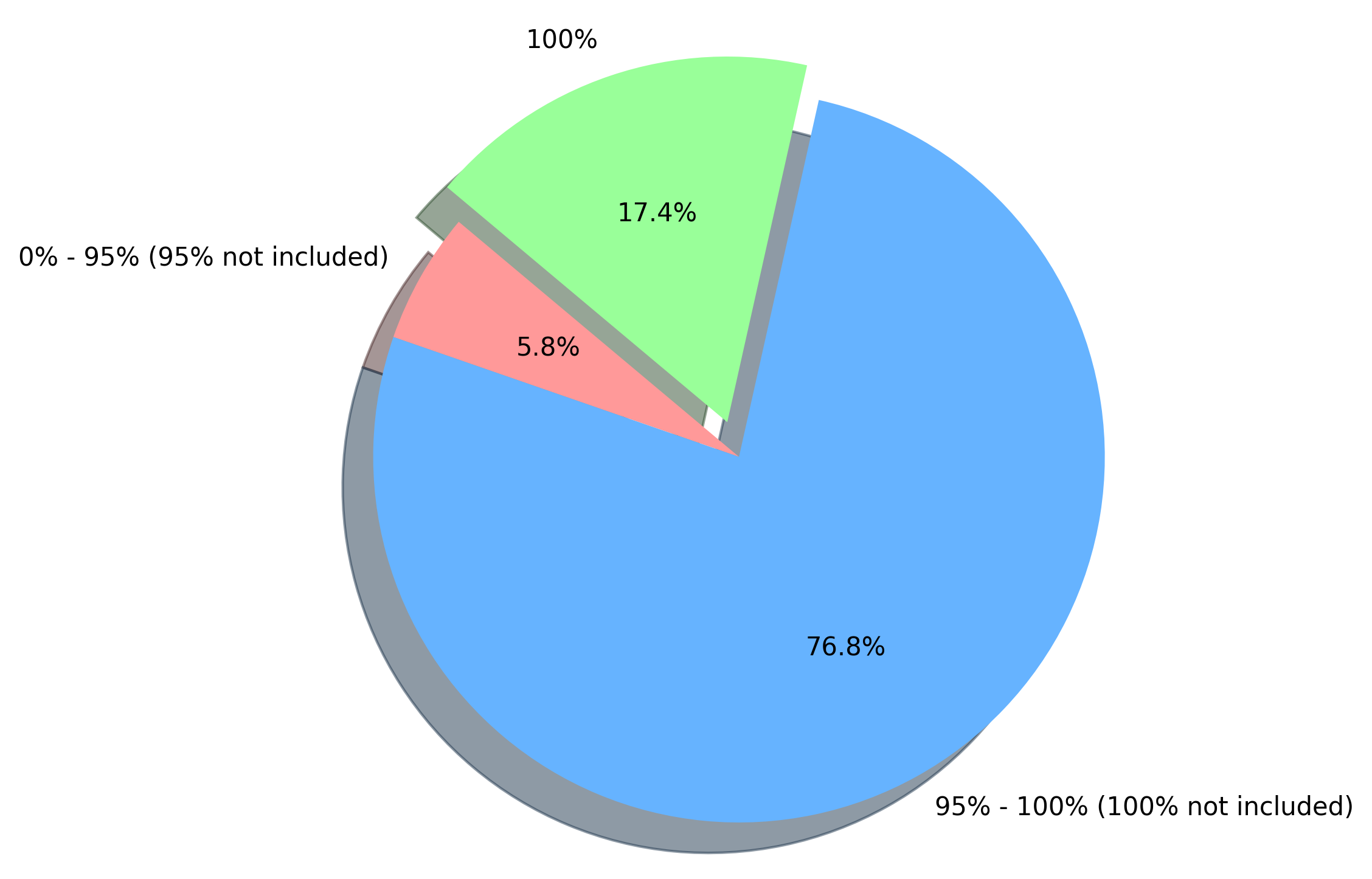}
        \caption{Distribution of valid proofs in ACAS Xu benchmark in quantization scenarios}
        \label{acasvalidity}
    \end{subfigure}
    \hfill
    \begin{subfigure}[b]{0.45\textwidth}
        \includegraphics[width=\textwidth]{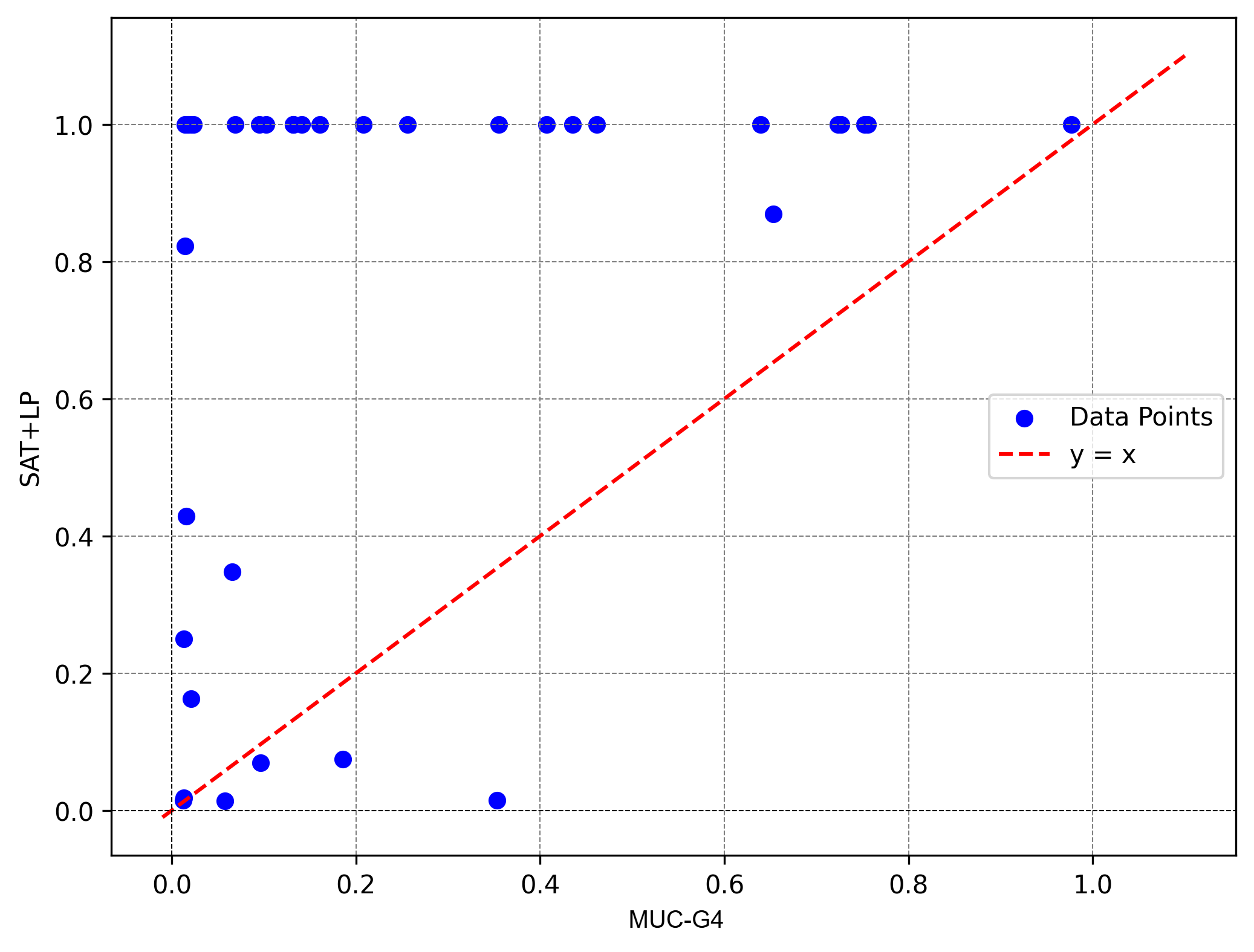}
        \caption{Running time comparison of vanilla SMT and MUC-G4 enhanced method on ACAS Xu benchmark.}
        \label{acastime}
    \end{subfigure}
    \caption{Empirical Evaluation on ACAS Xu Benchmark}
    \label{fig:combined}
\end{figure}

We conducted an evaluation of the validity of the generated proofs, i.e., the unsat cores in the original network verification. The results are presented in Fig.\ref{acasvalidity}. The high proportion (76.8\%) of unsat cores with validity between 95\% - 100\% (100\% not included) indicates that the MUC-G4 framework is generally effective in generating proofs that are highly reusable in the context of network compression for ACAS - Xu benchmarks. This suggests that in most cases, the generated unsat cores require minimal re - verification effort after network compression. The 17.4\% of unsat cores with 100\% validity are particularly significant, as they imply that for these cases, the generated proofs are directly applicable to the compressed networks without any modification, showcasing the robustness of the MUC-G4 approach. Besides, the 5.8\% of unsat cores with validity between 0 - 95\% (95\% not included) have no less than 90\% validity. Such results proves that most of the proofs that MUC-G4 provides can be reused in the incremental verification during network quantization.

We record the running time to evaluate the significance of the generated proofs. The results are illustrated in Fig.\ref{acastime}, which compares the total query-solving time for conventional SMT solving against MUC-G4. To facilitate interpretation of the results, we exclude instances where both methods exceed the time limit, and we compute the ratio of actual solving time to the time limit. From Fig.\ref{acastime}, it is evident that MUC-G4 achieves a degree of acceleration in the verification of compressed networks. Over 34.2\% of the previously unknown cases are resolved using MUC-G4. Additionally, the average speedup—calculated as the ratio of the verification time for conventional SMT solving to that of MUC-G4—is approximately 14.1 times. However, there are instances where MUC-G4 does not provide an acceleration, with the majority yielding a SAT result. Nevertheless, even in these cases, MUC-G4 demonstrates only minor performance degradation compared to the conventional SMT solving method.

% \begin{figure}[hbtp]
%     \centering
% \includegraphics[width=0.6\columnwidth]{Figure/acas_time.png}
%     \caption{Running time comparison of SMT and MUC-G4 enhanced method on ACAS Xu benchmark.}
%     \label{acastime}
% \end{figure}

\subsubsection{\textbf{MNIST}}
Based on the MNIST dataset, we trained 8 DNNs with architectures $2\times10$, $2\times15$, $2\times20$, $2\times50$, $2\times100$, $3\times15$, $6\times15$, $12\times15$, and prune their edges with a ratio of 0.2. Besides, we randomly sampled 4 data points with input perturbation $\epsilon=0.1$, yielding 32 incremental verification problems.

We conducted an evaluation of the validity of the generated proofs, i.e., the unsat cores in the original network verification. The results are presented in Fig.\ref{mnistvalidity}. The high proportion (47.2\%) of values within the range of 80\% - 100\% indicates that the MUC-G4 framework is generally effective in generating proofs that are highly valid in pruned network verification. This suggests that in most cases, the generated proofs can be reused after pruning.
The 16.7\% of values in the 0\% - 20\% range are significant as they imply that in these cases, the generated unsat cores are relatively not so robust. It could be due to the pruning of some crucial edges in the network.
The values in the ranges of 20\% - 40\% (11.1\%), 40\% - 60\% (13.9\%), and 60\% - 80\% (11.1\%) show a relatively balanced distribution. These intermediate-level values suggest that even with edge pruning, MUC-G4 can generate meaningful proofs.
Overall, such results prove that while a significant portion of the generated proofs by MUC-G4 are highly satisfactory, there are still cases where the proofs cannot be reused, indicating areas for potential improvement and further exploration on generating more ``robust'' proofs regarding network pruning.

\begin{figure}[t]
    \centering
    \begin{subfigure}[b]{0.45\textwidth}
        \includegraphics[width=\textwidth]{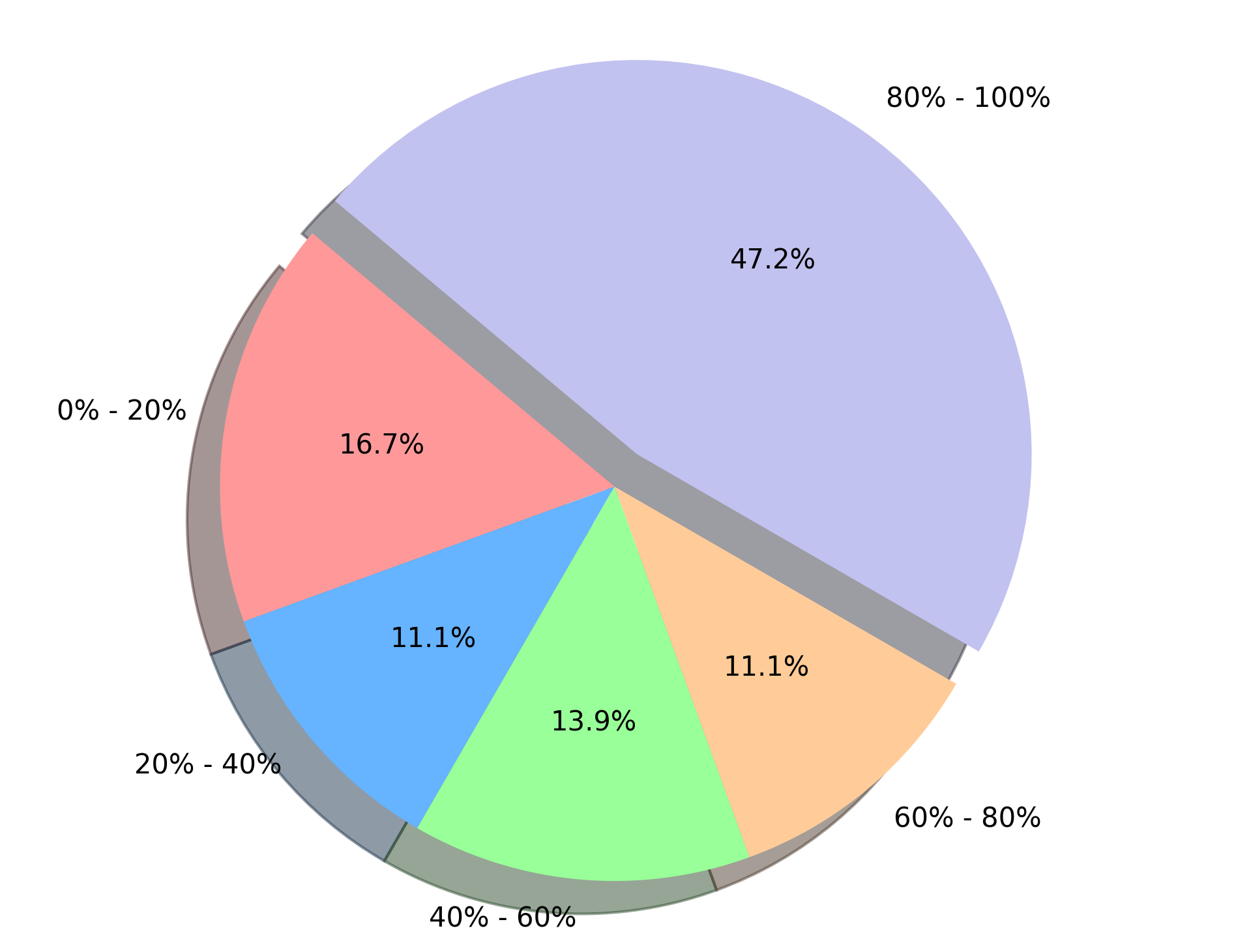}
        \caption{Distribution of valid proofs in MNIST Benchmark in Network Pruning Scenarios}
        \label{mnistvalidity}
    \end{subfigure}
    \hfill
    \begin{subfigure}[b]{0.45\textwidth}
        \includegraphics[width=\textwidth]{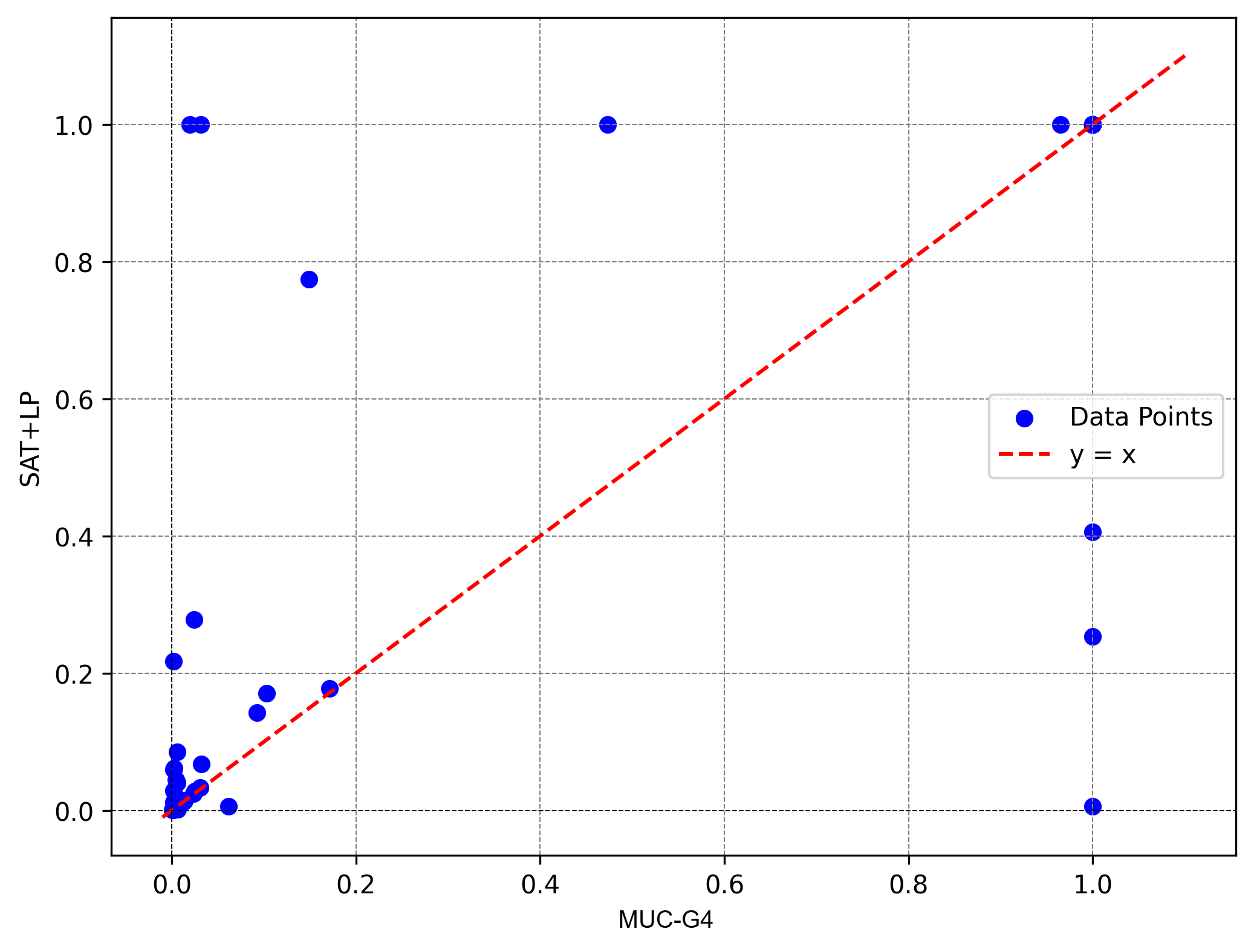}
        \caption{Running time comparison of vanilla SMT and MUC-G4 enhanced method on MNIST benchmark.}
        \label{mnisttime}
    \end{subfigure}
    \caption{Empirical Evaluation on MNIST Benchmark}
    \label{fig:mnist_combined}
\end{figure}
We record the running time for MNIST benchmark to evaluate the significance of the generated proofs in network pruning based incremental verification. 
The results are illustrated in Fig.\ref{mnisttime}, which compares the total query-solving time for conventional SMT solving against MUC-G4. 
To facilitate interpretation of the results, we exclude instances where both methods exceed the time limit, and we compute the ratio of actual solving time to the time limit. From Fig.\ref{mnisttime}, it is evident that MUC-G4 achieves acceleration in for the majority of pruned cases. Over 19.0\% of the previously unknown cases are resolved using MUC-G4. Additionally, the average speedup—calculated as the ratio of the verification time for conventional SMT solving to that of MUC-G4—is approximately 9.2 times. However, there are three instances where MUC-G4 perform much worse than vanilla SMT solving. In these cases, the proofs generated by the original networks cannot guide the search for the pruned network.

% \begin{figure}[hbtp]
%     \centering
% \includegraphics[width=0.6\columnwidth]{Figure/mnist_time.png}
%     \caption{Running time comparison of SMT and MUC-G4 enhanced method on MNIST benchmark.}
%     \label{mnisttime}
% \end{figure}

\subsubsection{\textbf{Analysis}}
We further discuss some of the possible findings in our experiments.

\smallskip\noindent\textit{Quantization vs. Pruning}
While network pruning can be viewed as a specific form of network quantization, we find that these two approaches differ in terms of the validity and effectiveness of the generated proofs (unsatisfiable cores). As anticipated, quantized models exhibit a relatively high acceptance of the unsatisfiable cores generated during the verification process of the original network. This suggests that the overall activation states of the ReLU neurons are well preserved. In contrast, although network pruning also maintains an acceptable level of validity for the generated proofs, it may compromise the internal similarity among neurons. This phenomenon can diminish the similarity between the unsatisfiable cores of the original and pruned networks, resulting in a more costly search process during verification. %During our experiments, as discussed in the previous section, we observed that pruning can sometimes compromise the structural information inherent in the activation states of ReLU neurons, particularly when compared to quantization. In contrast, SMT solving offers a more effective approach to pruning that preserves the validity of the generated proofs.

\smallskip\noindent\textit{Overall Performance}
We also notice that the overall performance may not be satisfactory. This is largely due to the running time being highly dependent on the chosen theory solver. We selected a conventional LP solver to eliminate potential hybrid effects that the searching strategies of existing neural network verifiers might introduce. However, it is possible to employ any neural network verifier as the theory solver for the subproblem generated given a search path, to achieve acceleration with MUC-G4. 

\section{Related Work} \label{relatedwork}
Two existing mainstream formal verification strategies for deep neural networks are constraint-based verification and abstraction-based verification~\cite{Albarghouthi2021IntroductionTN}. Constraint-based methods~\cite{Katz2019TheMF,Katz2017ReluplexAE} typically rely on constraint solvers, offering soundness and completeness guarantees but often require significant computational resources. Even for basic neural networks and properties, these methods can exhibit NP-completeness~\cite{Katz2017ReluplexAE}.
In contrast, abstraction-based verification approaches~\cite{2018AI,Singh2019AnAD,2018Fast} prioritize scalability over completeness. At a broader level, techniques such as interval analysis and bound propagation~\cite{2018Efficient,2018Formal} can be seen as specific forms of state abstraction.
CROWN~\cite{Zhang2018EfficientNN} has expanded the concept of bound propagation from value-based bounds to linear function-based bounds of neurons, making significant strides in scalability. However, it does not guarantee complete verification.
In~\cite{Wang2021BetaCROWNEB}, the Branch-and-Bound (BaB) technique and GPU acceleration are employed to maintain completeness while enhancing scalability in verification processes. This combination address the trade-off between completeness and scalability in formal verification of deep neural networks.

Conflict-Driven Clause Learning (CDCL) \cite{DBLP:series/faia/0001LM21} is a vital technique widely adopted in contemporary SAT and SMT solvers \cite{DBLP:journals/jsat/Sebastiani07,DBLP:journals/jacm/NieuwenhuisOT06}. This framework integrates conflict clause learning, unit propagation, and various strategies to enhance efficiency, proving effective in numerous branch-and-bound problems, including SAT and linear SMT. Recent work by \cite{DBLP:conf/tase/LiuYZH24} incorporates the algorithm into neural network verification, introducing DeepCDCL, which utilizes conflict clauses to prune the search space and expedite the solving process. Additionally, they propose an asynchronous clause learning and management structure that reduces time overhead during solving, thus accelerating the overall process. However, their approach primarily focuses on optimizing the search space within the context of the network's verification, without extending these advancements to enhance incremental verification in scenarios involving network compression.

Regarding the incremental verification in neural networks, several works that focus on this problem are worth mentioning. Ivan~\cite{DBLP:journals/pacmpl/UgareBM023} creates a specification tree, which is a new type of tree data structure that represents the trace of BaB. This tree is generated by running the complete verifier on the original network. DeepInc~\cite{DBLP:journals/corr/abs-2302-06455} simulates important parts of Reluplex solving process, aiming to determine the verification result of the branches in the original procedure, and uses a heuristic approach to check if the proofs are still valid for the modified DNN. ~\cite{DBLP:conf/cav/FischerSDSV22} introduced the idea of sharing certificates between specifications, leveraging proof reuse from $L_\infty$ specifications computed with abstract interpretation-based analyzers via proof templates. This approach facilitates faster verification of patch and geometric perturbations. Similarly, ~\cite{DBLP:journals/pacmpl/UgareSM22} demonstrated the feasibility of reusing proofs between networks, employing network adaptable proof templates on specific properties such as patch, geometric, and $L_0$. While these efforts advance incremental verification techniques, they do not address the challenge of incremental and complete verification with possibly altered network structure in network compression scenarios, which is the primary focus of our work.

\section{Conclusion} \label{conclusion}
In our research, we have formally defined the problem of incremental verification in version control within network compression-based software customization. This task becomes challenging due to the potential alterations in weights and network structure, which can be captured through Satisfiability Modulo Theories (SMT) formulas. To address the variability in neural compression techniques, we introduced MUC-G4, a method that guides the verification search with MUC-based proofs generated when encountering infeasible paths in the Theory solver, enabling a systematic approach for incrementally verifying compressed networks.

Moving forward, our future works will focus on several key areas. Firstly, we aim to enhance the efficiency of generating reusable proofs in parallel schemes to expedite the reusability process. Secondly, we plan to explore the extension of SMT-based proof reuse techniques to encompass a broader range of activation functions, such as hard Tanh and PReLU. Additionally, we intend to broaden the scope of assumptions related to changes in network structure to adapt our methodologies for incremental verification scenarios involving neuron augmentation in practical settings. Finally, we aspire to leverage the generated proofs within the SMT tree to guide the network compression process, thereby producing models that are more secure and resilient.

%
% ---- Bibliography ----
%
% BibTeX users should specify bibliography style 'splncs04'.
% References will then be sorted and formatted in the correct style.
%
\bibliographystyle{splncs04}
\bibliography{mybibliography}
%
% \begin{thebibliography}{8}
% \bibitem{ref_article1}
% Author, F.: Article title. Journal \textbf{2}(5), 99--110 (2016)

% \bibitem{ref_lncs1}
% Author, F., Author, S.: Title of a proceedings paper. In: Editor,
% F., Editor, S. (eds.) CONFERENCE 2016, LNCS, vol. 9999, pp. 1--13.
% Springer, Heidelberg (2016). \doi{10.10007/1234567890}

% \bibitem{ref_book1}
% Author, F., Author, S., Author, T.: Book title. 2nd edn. Publisher,
% Location (1999)

% \bibitem{ref_proc1}
% Author, A.-B.: Contribution title. In: 9th International Proceedings
% on Proceedings, pp. 1--2. Publisher, Location (2010)

% \bibitem{ref_url1}
% LNCS Homepage, \url{http://www.springer.com/lncs}, last accessed 2023/10/25
% \end{thebibliography}

\appendix
\section{MUC Guided Incremental Verification}\label{algo}
\begin{algorithm}[H]
    
    \caption{MUC Guided Incremental Verification for DNN Compression}
    \label{algo:main}
    \textbf{Input:} Proof $p(f)$ of a verification problem $(f,P,Q)$,

    the compressed network $f'$

    \textbf{Output:} SAT if $(f',P,Q)$ does not hold, UNSAT otherwise.
    
    \If {Check($f$) = SAT}{
        \If{Check(sat\_core($f$)) = SAT}{
            return SAT
        }
        \Else{
            \For{unchecked\_core($f$) $c_{\text{unk}}$ in $p(f)$}{
                \If{Check($c_{\text{unk}}$) = SAT}{
                    return SAT
                }
            }
            \For{unsat\_core($f$) $c_{\text{unsat}}$ in $p(f)$}{
                \If{Check($c_{\text{unsat}}$) = SAT}{
                    \For{possible branch $b$ containing $c_{\text{unsat}}$}{
                        \If{Check($b$) = SAT}{
                            return SAT
                        }
                    }
                }
            }
            return UNSAT
        }
    }
    \Else{
        \For{unchecked\_core($f$) $c_{\text{unk}}$ in $p(f)$}{
            \If{Check($c_{\text{unk}}$) = SAT}{
                return SAT
            }
        }
        \For{unsat\_core($f$) $c_{\text{unsat}}$ in $p(f)$}{
            \If{Check($c_{\text{unsat}}$) = SAT}{
                \For{possible branch $b$ containing $c_{\text{unsat}}$}{
                    \If{Check($b$) = SAT}{
                        return SAT
                    }
                }
            }
        }
        return UNSAT
    }
    
\end{algorithm}

\end{document}